\documentclass[11pt]{article}

\usepackage{microtype}
\usepackage{graphicx}
\usepackage{hyperref}

\usepackage{times}
\usepackage{amsthm,amssymb,amsmath}
\usepackage{algorithm,algpseudocode}
\usepackage{url}
\usepackage{authblk}

\theoremstyle{definition}
\newtheorem{theorem}{Theorem}  [section] 

\newtheorem{proposition}[theorem] {Proposition}

\newtheorem{problem}[theorem] {Problem}

\numberwithin{equation}{section}

\DeclareMathOperator*{\argmax}{argmax}

\begin{document}

\title{Maximizing Invariant Data Perturbation with Stochastic Optimization}

\author{Kouichi Ikeno}
\author{Satoshi Hara}
\affil{Osaka University, Japan\\ \url{{k1keno, satohara}@ar.sanken.osaka-u.ac.jp}}

\date{}

\maketitle

\begin{abstract}
Feature attribution methods, or saliency maps, are one of the most popular approaches for explaining the decisions of complex machine learning models such as deep neural networks.
In this study, we propose a stochastic optimization approach for the perturbation-based feature attribution method.
While the original optimization problem of the perturbation-based feature attribution is difficult to solve because of the complex constraints, we propose to reformulate the problem as the maximization of a differentiable function, which can be solved using gradient-based algorithms.
In particular, stochastic optimization is well-suited for the proposed reformulation, and we can solve the problem using popular algorithms such as SGD, RMSProp, and Adam.
The experiment on the image classification with VGG16 shows that the proposed method could identify relevant parts of the images effectively.
\end{abstract}

\section{Introduction}
\label{sec:intro}

Feature attribution methods~\cite{simonyan2013deep,springenberg2014striving,bach2015pixel,sundararajan2017axiomatic,smilkov2017smoothgrad,shrikumar2017learning,hara2018maximally}, or saliency maps, are one of the most popular approaches for explaining the decisions of complex machine learning models such as deep neural networks.
In feature attribution, for each of given instance, they score how much each feature is relevant to the model's decision: they score the features relevant to the model's decision with large values and irrelevant features with small values.
For example, in image recognition, they highlight which pixels the models have focused on by scoring the relevance of each pixel~\cite{simonyan2013deep,springenberg2014striving,bach2015pixel,sundararajan2017axiomatic,smilkov2017smoothgrad}, and in text classification, they detect a set of words or sentences relevant to the model's decision by scoring each words or sentences~\cite{ding2017visualizing,chen2018learning}.
With feature attribution methods, one can obtain relevant features, such as pixels or words, as explanations why the models made certain decisions, which also helps the users to inspect whether the models are reliable or not.
Major approaches for feature attribution are based on modified gradients~\cite{simonyan2013deep,springenberg2014striving,bach2015pixel,sundararajan2017axiomatic,smilkov2017smoothgrad,shrikumar2017learning} and feature maskings~\cite{fong2017interpretable}.

Recently, Hara et al.~\cite{hara2018maximally} proposed to separate the definition of the feature attribution scores and the algorithms to compute them.
In most of the previous studies, scores are defined by their computation algorithms themselves except for some axiomatic approaches~\cite{sundararajan2017axiomatic,lundberg2017unified}.
The separation of the definitions and the algorithms allows us to consider each aspect independently.
For example, if the definition is not appropriate, improving algorithms does not help and we need to reconsider the definition in such a situation.

As one possible definition of the feature attribution score, Hara et al.~\cite{hara2018maximally} proposed to measure the irrelevance of each feature to the model's decision by the maximum size of data perturbations that does not change the decision.
Specifically, they defined the score as a solution to the optimization problem that maximizes the size of data perturbation under the constraint that the perturbed data to remain inside the model's decision boundary.
They then proposed an algorithm that solves the optimization problem approximately: they approximated the constraint with linear functions and reformulated the problem as linear programming.
The linear programming formulation is found to be useful with several flexible extensions such as relaxing constraints and sharing scores among features.
The drawback of the linear approximation, however, is that we cannot strictly enforce the perturbed data to stay inside the model's decision boundary.
Therefore, the solution to the linear programming can violate the definition.

This study is positioned as the improvements of the algorithm for the score defined by Hara et al.~\cite{hara2018maximally}.
Specifically, to relieve the risk of the definition violation in the method of Hara et al.~\cite{hara2018maximally}, we propose an algorithm to solve the problem \emph{without} linear approximation.
In the proposed approach, we reformulate the problem so that it can be solved by using gradient-based algorithms.
Specifically, we rewrite the constraint as a differential penalty function in the objective function to be maximized.
With this reformulation, the problem is expressed as the maximization of a differentiable function, which can be solved using gradient-based algorithms.
In particular, stochastic optimization is well-suited for the proposed reformulation, and we can solve the problem using popular algorithms such as SGD, RMSProp~\cite{tieleman2012lecture}, and Adam~\cite{kingma2014adam}.

\paragraph{Settings}
In this paper, we consider the classification model $f$ for $k$ categories that returns an output $y \in \mathbb{R}^k$ for a given input $x \in \mathbb{R}^d$, i.e., $y = f(x)$.
The classification result is determined as $c = \argmax_j y_j$ where $y_j = f_j(x)$ is the $j$-th element of the output.
We assume that the model $f$ is differentiable with respect to the input $x$: the target models therefore include linear models, kernel models with differentiable kernels, and deep neural networks.
We assume that the model $f$ and the target input $x$ to be explained are given and fixed.

\section{Problem Definition: Maximizing Invariant Data Perturbation}
\label{sec:problem}

In this section, we briefly review the definition of the problem introduced by Hara et al.~\cite{hara2018maximally}.
As an explanation of the input $x$, we seek the maximally invariant data perturbation that does not change the model's decision.

We start from introducing \emph{invariant perturbation set}.
We say that a set $R \subseteq \mathbb{R}^d$ is an invariant perturbation set if the model's decision is invariant for all $r \in R$, i.e., $c = \argmax_j f_j(x+r)$.

In the study, for ease of computation, we restrict our attention to a box-shaped invariant perturbation set $R(w) = [-w_1, w_1] \times [-w_2, w_2] \times \cdots \times [-w_d, w_d]$ for a parameter $w \in \mathbb{R}^d_+$\footnote{In Hara et al.~\cite{hara2018maximally}, the lower and upper boundaries of the box are parametrized by different parameters $u$ and $v$. Here, we use a common parameter $w$ for ease of computation.}. 
From this definition of $R(w)$, the size of the invariant perturbation of each feature $x_i$ is proportional to $w_i$.
The idea here is that, if the invariant perturbation $w_i$ is small, the change of the feature $x_i$ can highly impacts the model's decision, which indicates that the feature $x_i$ is relevant to the decision.
On the other hand, if $w_{i'}$ is large, the feature $x_{i'}$ only has a minor impact to the model's decision, and thus it is less relevant.

To obtain an invariant perturbation set $R(w)$ appropriate for feature attribution, Hara et al.~\cite{hara2018maximally} proposed to maximize the side lengths of the $R(w)$ so that $w_i$ to be sufficiently large for irrelevant features.
\begin{problem}[Maximal Invariant Perturbation]
	\label{prob:problem}
	Find the invariant perturbation set $R(\hat{w})$, where
	\begin{align}
		\hat{w} = \argmax_{w \in \mathbb{R}_+^d, \|w\|_\infty \le C} \sum_{i=1}^d w_i, \; {\rm s.t.} \; c = \argmax_j f_j(x + r), \forall r \in R(w) .
		\label{eq:problem}
	\end{align}
\end{problem}
Here, $C > 0$ is the upper bound of the perturbation.
The value of $C$ can be usually determined from the nature of the data.
For example, if the data is the image, the value of each pixel is usually restricted in $[0, 1]$.
In this case, the upper bound $C = 1$ is the natural choice.

\section{Proposed Method}
\label{sec:method}

We now turn to our proposed method to solve the problem (\ref{eq:problem}).
The difficulty on solving the problem (\ref{eq:problem}) is that it requires the constraint $c = \argmax_j f_j(x + r)$ to hold for all possible $r \in R(w)$.
In the proposed method, we rewrite the constraint by using the expectation over $r \in R(w)$.
In this way, we can reformulate the problem as the maximization of a differentiable function, which can be solved using gradient-based algorithms.

\subsection{Problem Reformulation}

We reformulate the problem (\ref{eq:problem}) into the penalty-based expression, so that the gradient-based optimization algorithms to be applicable.
Recall that the constraint $c = \argmax_j f_j(x + r)$ is equivalent to a set of constraints $f_c(x+r) \ge f_j(x + r)$ over $\forall j \neq c$.
Here, we rewrite the constraint $f_c(x+r) \ge f_j(x + r)$ into the following equivalent expression using expectation.
\begin{proposition}
	The constraint $f_c(x+r) \ge f_j(x + r), \forall r \in R(w)$ is equivalent to
	\begin{align}
		\mathbb{E}_t[\max\{0, f_j(x+t \odot w) - f_c(x + t \odot w)\}] \le 0, 
		\label{eq:exp_constraint}
	\end{align}
	where $\mathbb{E}_t$ denotes the expectation over uniformly random $t \in [-1, 1]^d$, and $\odot$ denotes an element-wise product.
\end{proposition}
\begin{proof}
	We first note that $r := t \odot w$ is a uniform random variable in $R(w)$.
	Here, recall that $\max\{0, f_j(x+r) - f_c(x + r)\}$ is non-negative.
	Therefore, the inequality $\max\{0, f_j(x+r) - f_c(x + r)\} \le 0$ holds if and only if $f_c(x+r) \ge f_j(x + r)$.
	Hence, for the inequality $\mathbb{E}_t[\max\{0, f_j(x+r) - f_c(x + r)\}] \le 0$ to hold, the measure of the points that violates the inequality, i.e., $\{r \in R(w) : f_c(x+r) < f_j(x + r)\}$, must be zero.
\end{proof}

We now put the constraint (\ref{eq:exp_constraint}) as a penalty term in the objective function of (\ref{eq:problem}):
\begin{align}
	\begin{split}
	\hat{w} =& \argmax_{w \in \mathbb{R}_+^d} \sum_{i=1}^d w_i - \lambda \sum_{j \neq c} \mathbb{E}_t[h_{jc}(x + t \odot w)],
	\label{eq:problem2}
	\end{split}
\end{align}
where $\lambda \ge 0$ is a penalty parameter, and $h_{jc}(x + t \odot w) := \max\{0, f_j(x+t \odot w) - f_c(x + t \odot w)\}$.
We note that, in the problem (\ref{eq:problem2}), the constraint (\ref{eq:exp_constraint}) holds for sufficiently large $\lambda$, and hence the problem (\ref{eq:problem2}) is equivalent to (\ref{eq:problem}).

\subsection{Stochastic Optimization}

Thanks to the formulation (\ref{eq:problem2}), it is differentiable with respect to the parameter $w$.
Hence, we can apply the gradient-based optimization methods.
Here, we note that stochastic optimization algorithms, such as SGD, RMSProp~\cite{tieleman2012lecture}, and Adam~\cite{kingma2014adam}, are particularity suited for the problem.
In each iteration of the algorithm, we randomly generate perturbations $\{t_m : t_m \sim \mathrm{Uniform}([-1, 1]^d)\}_{m=1}^M$, and approximate the objective function using the sample average:
\begin{align}
	\sum_{i=1}^d w_i - \lambda \sum_{j \neq c} \frac{1}{M} \sum_{m=1}^M h_{jc}(x + t_m \odot w) .
	\label{eq:approx}
\end{align}
The psuedo code of the stochastic optimization is shown in Algorithm~\ref{alg:sgd}.
The line 6 is added in the algorithm so that the parameter $w$ to stay between the lower bound zero and the upper bound $C$.

\begin{algorithm}[tb]
\begin{algorithmic}[1]
\State{Initialize $w$ with a zero vector}
\While{not converged}
\State{Sample $\{t_m : t_m \sim \mathrm{Uniform}([-1, 1]^d)\}_{m=1}^M$}
\State{Compute the gradient of (\ref{eq:approx})}
\State{Update $w$ using an update rule, such as SGD, RMSProp, and Adam}
\State{Clip $w$ as $w = \min\{\max\{0, w\}, C\}$}
\EndWhile
\State{Return $w$}
\end{algorithmic}
\caption{Stochastic Optimization}
\label{alg:sgd}
\end{algorithm}

\section{Experiment}
\label{sec:experiment}

\subsection{Experimental Setup}

In the experiment, we followed the setup used in Hara et al.~\cite{hara2018maximally}.
As the target model $f$ to be explained, we adopted the pre-trained VGG16~\cite{simonyan2014very} distributed at the Tensorflow repository.
As the target data $x$ to be explained, we used COCO-animal dataset\footnote{\url{cs231n.stanford.edu/coco-animals.zip}}.
Specifically, for the experiment, we used $200$ images in the validation set.


In the proposed method, we set $\lambda = 0.1d$ and $M=32$ where the data dimension is $d =  224 \times224 \times 3$~\footnote{A sample code is available at \url{https://github.com/sato9hara/PertMap}}.
As the optimization algorithm, we used Adam~\cite{kingma2014adam} with the step size set to $0.01$ and remaining parameters set to be default values.
To measure the relevance of the feature $x_i$ to the model's decision, we used $-\hat{w}_i$, the negative of the perturbation size, as the score.
Large score, or small $\hat{w}_i$, means that the change of the feature $x_i$ can highly impacts the model's decision, which indicates that the feature $x_i$ is relevant to the decision.
On the other hand, if the score is small, or $w_{i'}$ is large, the feature $x_{i'}$ only has a minor impact to the model's decision, and thus it is less relevant.

As the baseline methods, we used Gradient~\cite{simonyan2013deep}, GuidedBP~\cite{springenberg2014striving}, SmoothGrad~\cite{smilkov2017smoothgrad}, IntGrad~\cite{sundararajan2017axiomatic}, LRP~\cite{bach2015pixel}, DeepLIFT~\cite{shrikumar2017learning}, Occlusion, and the LP-based methods (LP and LP(Smooth))~\cite{hara2018maximally}.
Gradient, GuidedBP, SmoothGrad and IntGrad are implemented using \texttt{saliency}\footnote{\url{https://github.com/PAIR-code/saliency}} with default settings, and LRP, DeepLIFT, and Occlusion are implemented using \texttt{DeepExplain}\footnote{\url{https://github.com/marcoancona/DeepExplain}} where we set the mask size for Occlusion as $(12, 12, 3)$.
For the LP-based methods, we used the same setting as the ones used in Hara et al.~\cite{hara2018maximally} ($\delta = 0.1$ and $\lambda \approx 0.5$).
We note that, all the methods including the proposed method return the scores of size $(224, 224, 3)$.

\subsection{Result}

We evaluated the effectiveness of each method by making pixels of the images.
Specifically, we mask low score pixels with gray colors and observe whether the model's classification result is resistant to the flipping.
We expect that good feature attribution methods to identify relevant pixels with high scores.
Therefore, with good attribution methods, the model's classification result will kept unchanged even if we flip many low score pixels to gray, as relevant parts of the images remain unflipped.
On the other hand, if the attribution methods fail to identify relevant pixels with high scores, the classification result can change even with a small number of flips.

We conducted the experiment as follows.
\begin{enumerate}
	\item Flip pixels with scores smaller than the $\tau$\% quantile to $0.5$ (i.e., we replace the selected pixels with gray pixels\footnote{We also conducted experiments by replacing with zero (black) or one (white). The results were similar, and thus omitted.}).
	\item Observe the ratio of the images with the classification result changes within the 200 images: the result changes in less images indicate that the feature attribution methods successfully identified relevant parts of the images.
\end{enumerate}

We varied the threshold quantile from $\tau = 0$ (no flip)  to $\tau = 100$ (all filp), and summarized the result in \figurename~\ref{fig:resvgg}.
It is clear that the proposed method is the most resistant to the masking: the changes on the classification results are kept almost zeros even if 50\% of the pixels are masked.
This indicates that the proposed method successfully identified relevant parts of the images.
Moreover, the proposed method consistently outperformed the LP-based methods.
We conjecture that the LP-based methods can violate the constraint in the problem because of approximation, which led to less accurate estimate of the solution to the problem (\ref{eq:problem}).
By contrast, the proposed method solves the problem \emph{without} approximation.
Hence, the derived solution is more accurate than the ones of the LP-based methods.

\begin{figure}[t]
	\centering
	\includegraphics[width=0.7\textwidth]{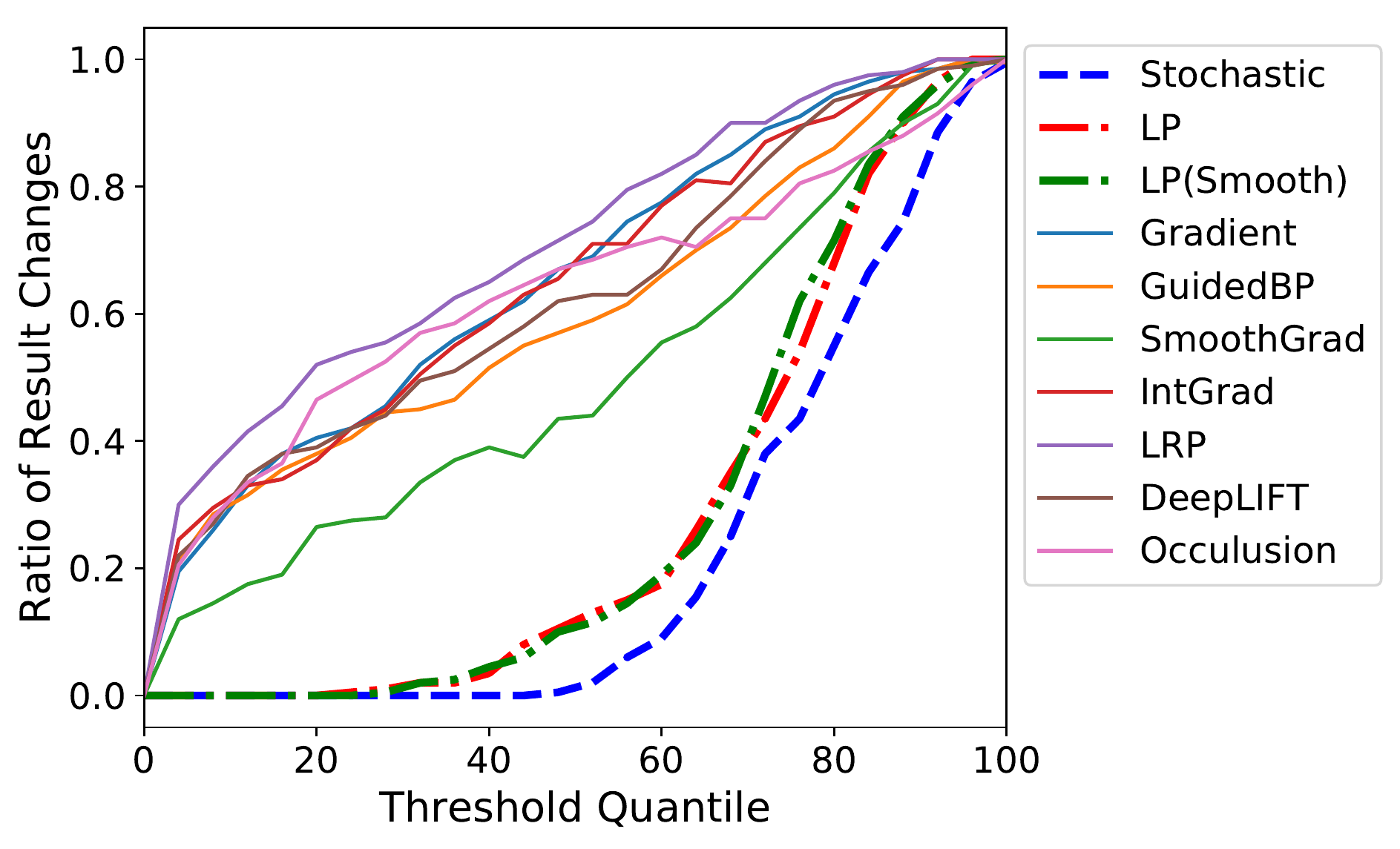}
	\caption{The ratio of the classification result changes when high score image patches are flipped to gray. \texttt{Stochastic} denotes the proposed method.}
	\label{fig:resvgg}
\end{figure}

\begin{figure}[!ht]
	\centering
	\includegraphics[width=0.99\textwidth]{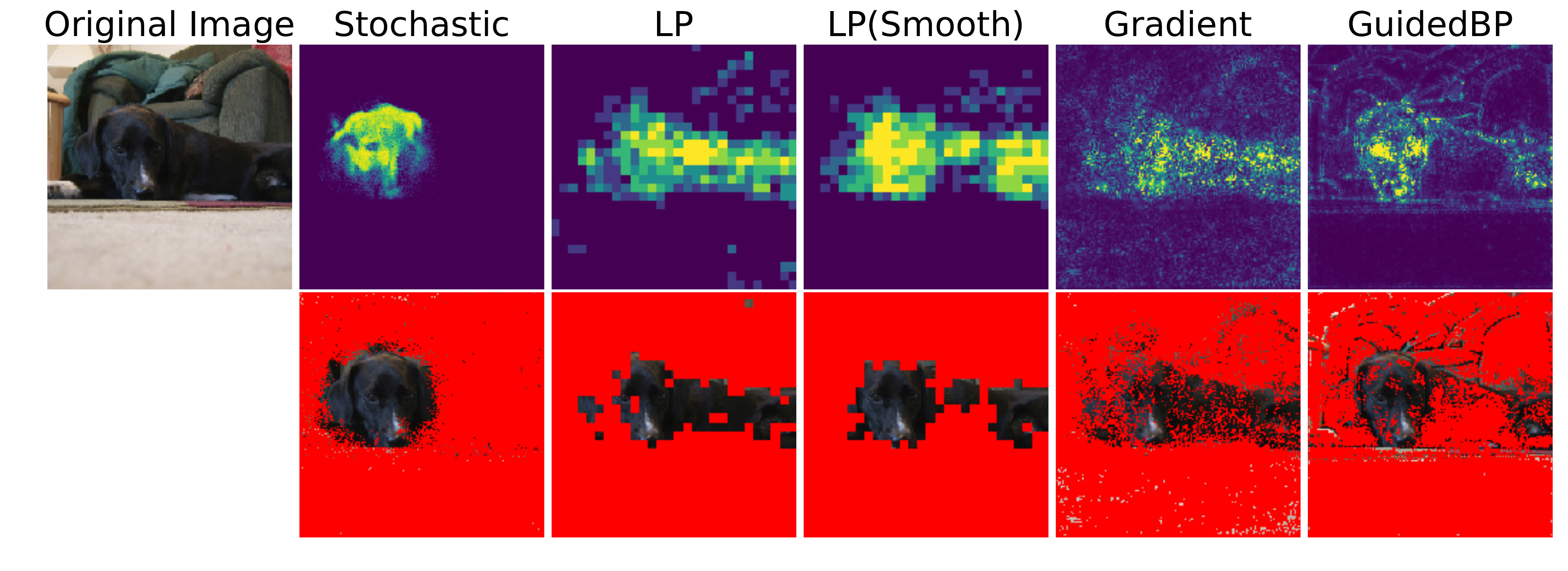}\\
	\hspace{0.15\textwidth}
	\includegraphics[width=0.83\textwidth]{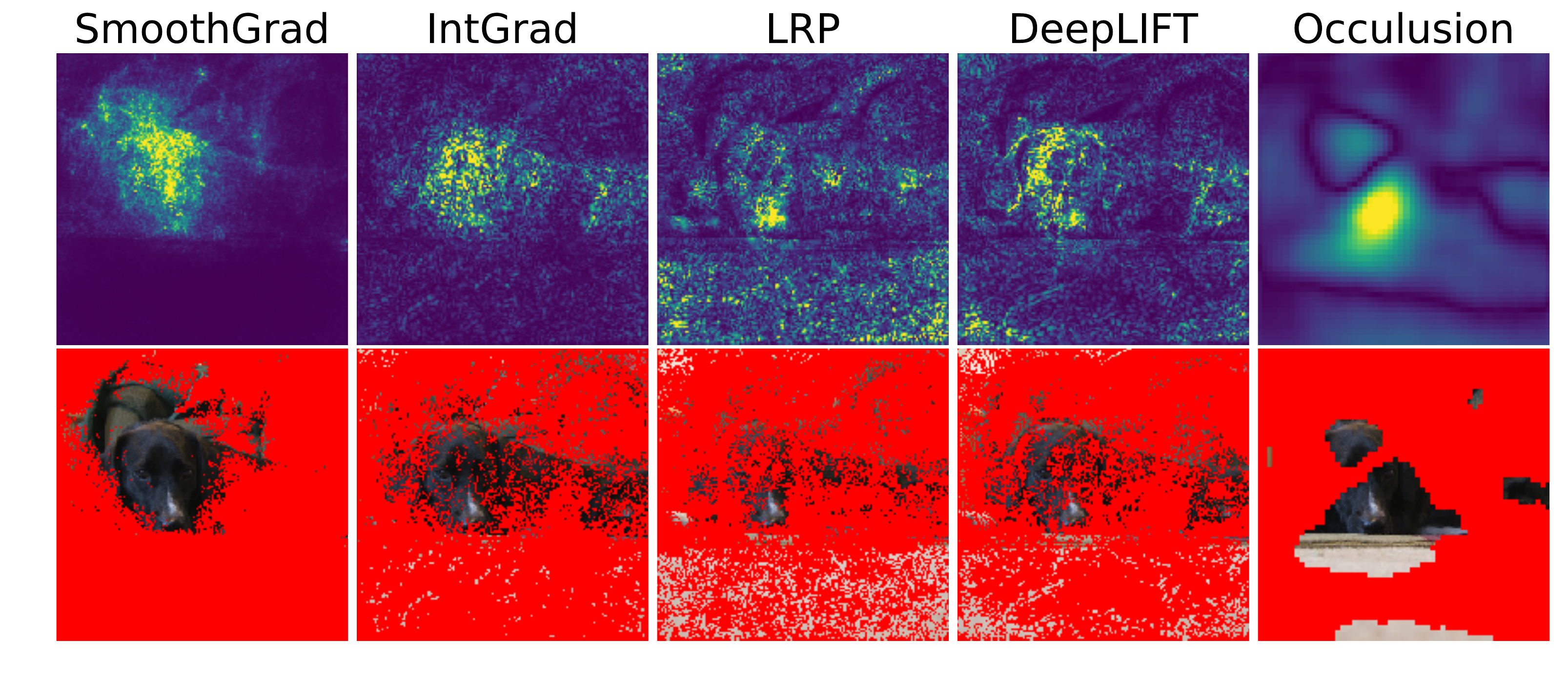}
	\caption{Example images (dog): the computed scores as heatmap (top); the filtered images with pixels with scores lower than the 80\% quantiles are filtered out with red (bottom). \texttt{Stochastic} denotes the proposed method.}
	\label{fig:example_dog}
\end{figure}

\begin{figure}[!ht]
	\centering
	\includegraphics[width=0.99\textwidth]{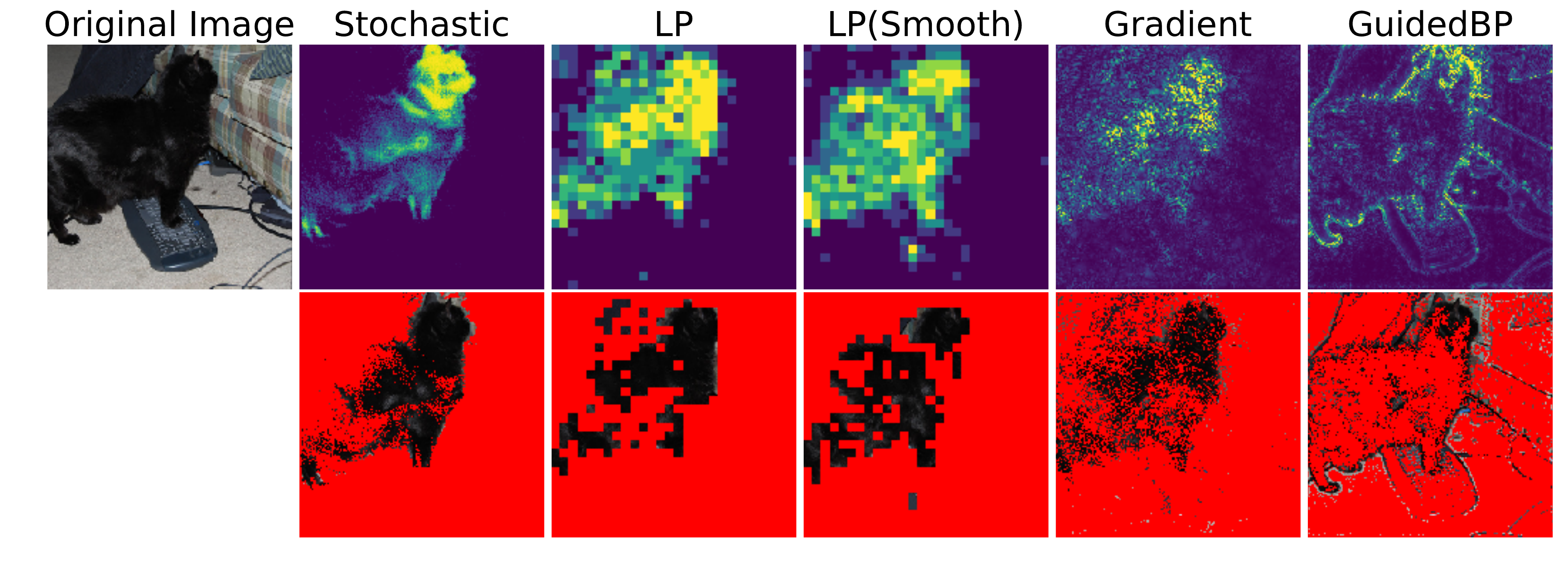}\\
	\hspace{0.15\textwidth}
	\includegraphics[width=0.83\textwidth]{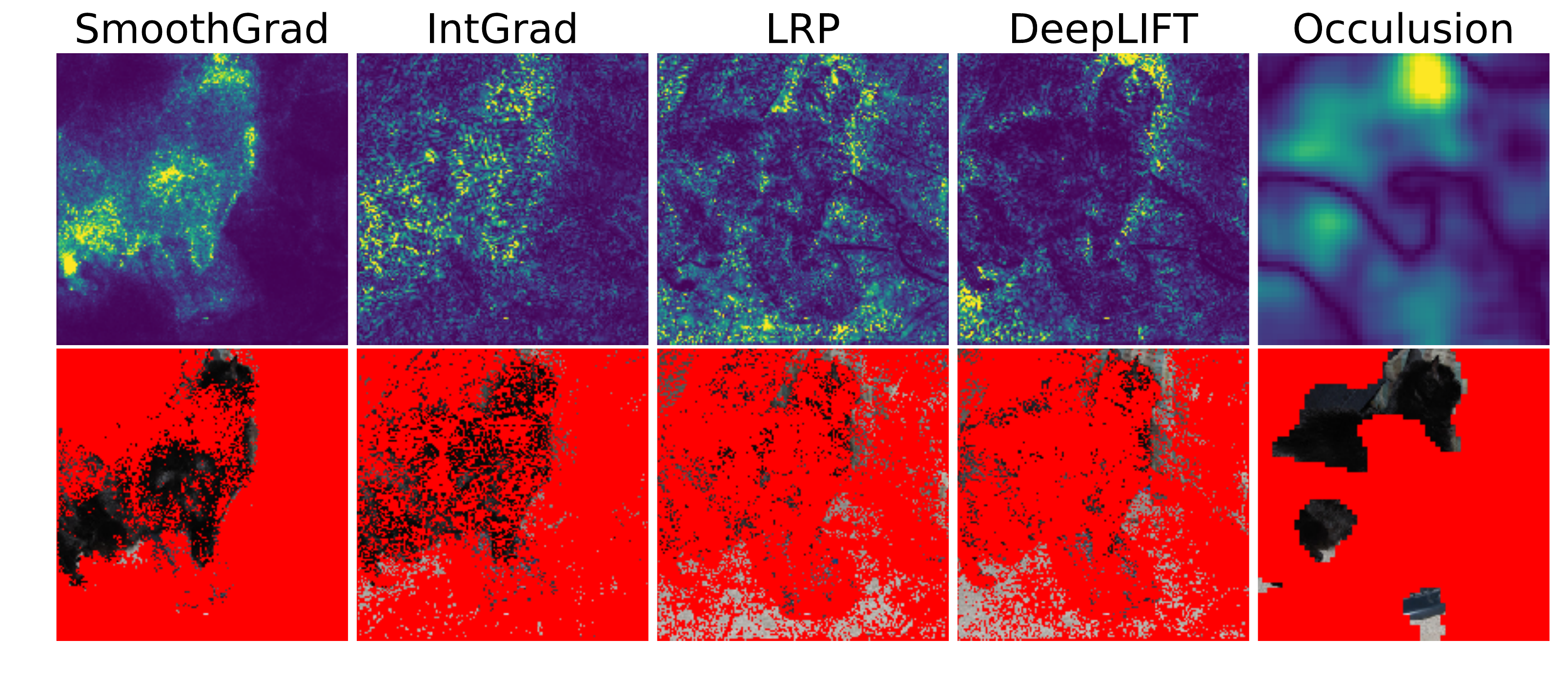}
	\caption{Example images (cat): the computed scores as heatmap (top); the filtered images with pixels with scores lower than the 80\% quantiles are filtered out with red (bottom). \texttt{Stochastic} denotes the proposed method.}
	\label{fig:example_cat}
\end{figure}

\begin{figure}[!ht]
	\centering
	\includegraphics[width=0.99\textwidth]{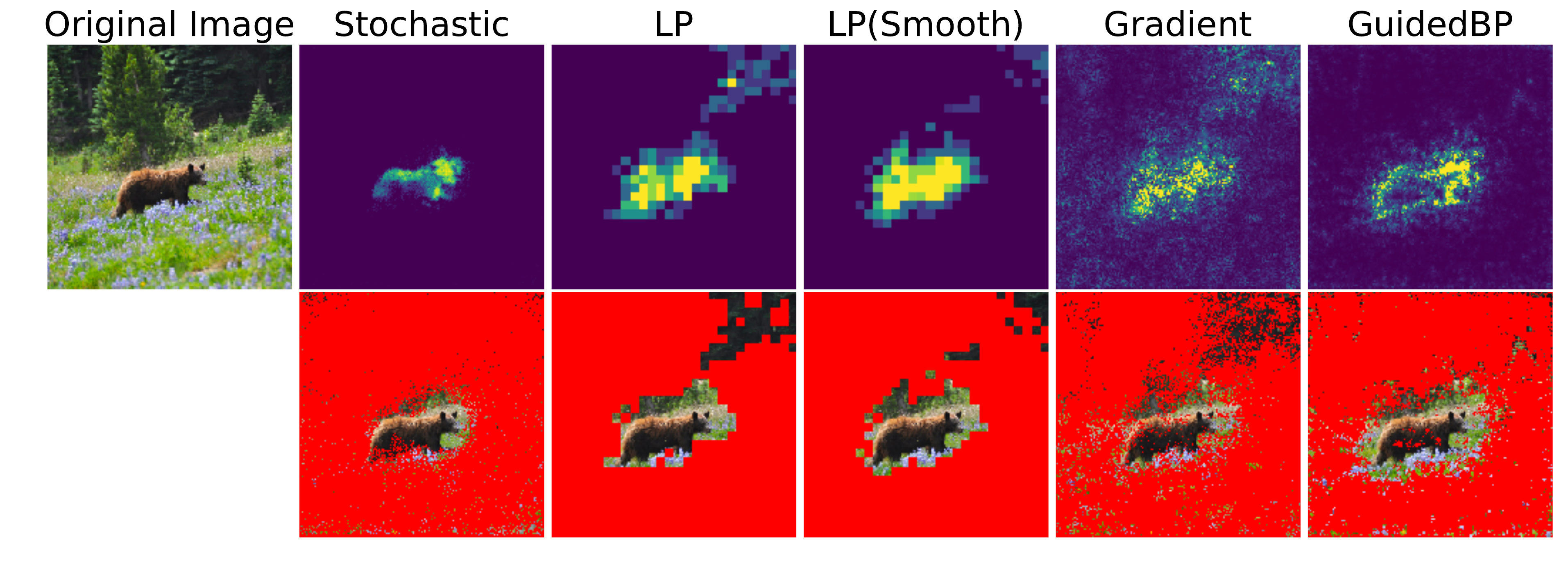}\\
	\hspace{0.15\textwidth}
	\includegraphics[width=0.83\textwidth]{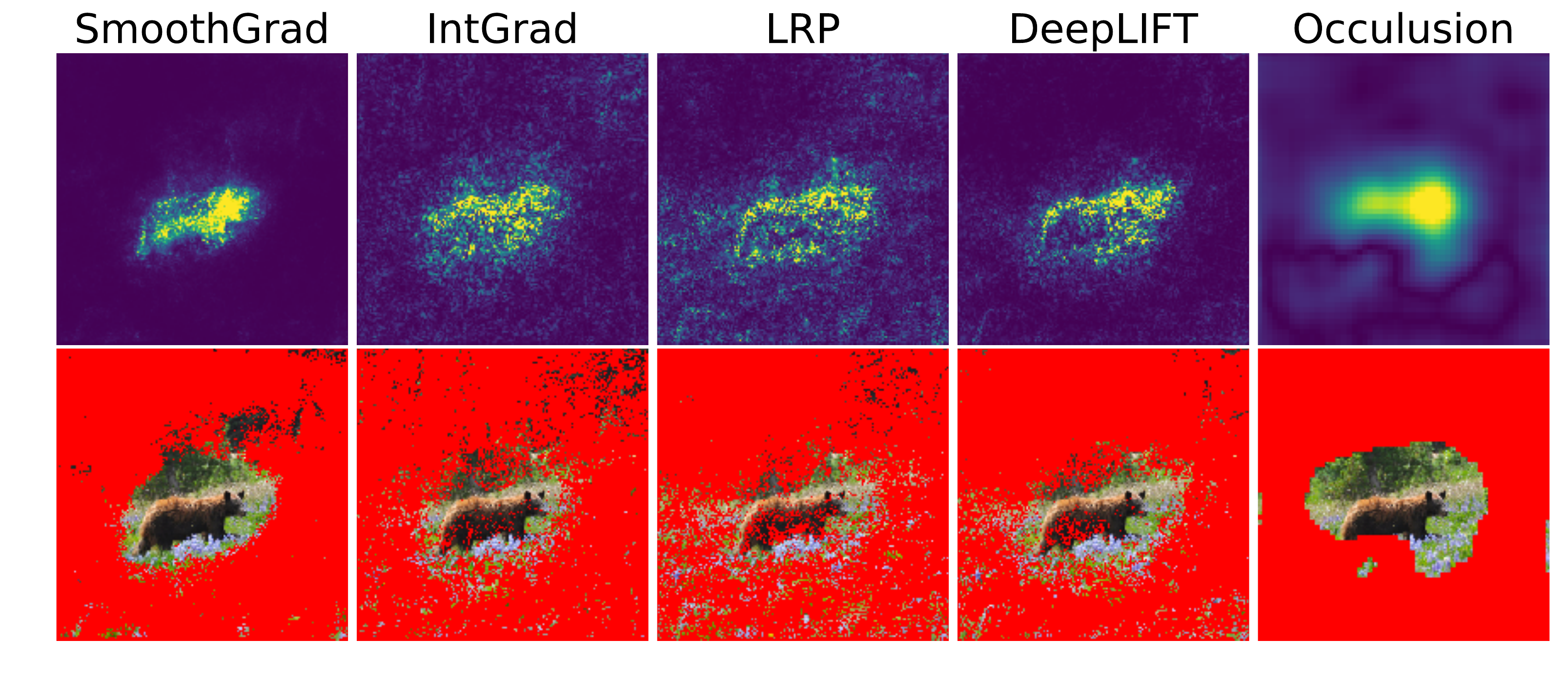}
	\caption{Example images (bear): the computed scores as heatmap (top); the filtered images with pixels with scores lower than the 80\% quantiles are filtered out with red (bottom). \texttt{Stochastic} denotes the proposed method.}
	\label{fig:example_bear}
\end{figure}

Figures~\ref{fig:example_dog}, \ref{fig:example_cat}, and \ref{fig:example_bear} show examples of the computed scores with each feature attirbution method.
The figures clearly show that the proposed method attained a very high S/N ratios compared to the existing method where the scores tend to be noisy.
We conjecture that this was because the proposed method could provide high quality solutions to the problem (\ref{eq:problem}) \emph{without} linear approximation.

\section{Conclusion}
\label{sec:conclusion}

In this study, as a novel feature attribution method, we proposed a stochastic optimization method for finding maximally invariant data perturbation.
In the proposed approach, we reformulated the problem as the maximization of a differentiable function, which can be solved using gradient-based algorithms.
In particular, stochastic optimization is well-suited for the proposed reformulation, and we can solve the problem using popular algorithms such as SGD, RMSProp, and Adam.
The experimental result on the image classification with VGG16 shows that the proposed method could identify relevant parts of the images effectively.

\section*{Acknowledgement}
This work was supported by JSPS KAKENHI Grant Number JP18K18106.

\bibliographystyle{unsrt}
\bibliography{main}

\end{document}